\PassOptionsToPackage{table}{xcolor}
\documentclass[11pt]{article}

\usepackage{fullpage}
\usepackage{amsmath}
\usepackage{amsfonts}
\usepackage{amsthm}
\usepackage{subfigure}
\usepackage{wrapfig,lipsum}
\usepackage[toc, page, header]{appendix}
\usepackage{titletoc}
\usepackage{fancyvrb}
\usepackage{algorithm}
\usepackage{algorithmic}
\usepackage{graphics}
\usepackage{enumitem}
\usepackage{graphicx}
\usepackage{hyperref}
\usepackage[table]{xcolor}
\definecolor{citecolor}{HTML}{0071BC}
\definecolor{linkcolor}{HTML}{ED1C24}
\definecolor{mydarkblue}{rgb}{0,0.08,0.45}
\hypersetup{colorlinks=true, linkcolor=mydarkblue, citecolor=mydarkblue,urlcolor=mydarkblue}
\usepackage{smile}
\usepackage{palatino}

\ifdefined\final
\usepackage[disable]{todonotes}
\else
\usepackage[textsize=tiny]{todonotes}
\fi
\title{
\huge RelationMatch: Matching In-batch Relationships for Semi-supervised Learning
}

\author{Yifan Zhang\thanks{Equal contribution} \thanks{
IIIS, Tsinghua University, e-mail: {\tt zhangyif21@mails.tsinghua.edu.cn}
} 
	~~
Jingqin Yang\footnotemark[1] \thanks{
IIIS, Tsinghua University, e-mail: {\tt yangjq21@mails.tsinghua.edu.cn}
} 
	~~
Zhiquan Tan\footnotemark[1] \thanks{
Department of Math, Tsinghua University, e-mail: {\tt tanzq21@mails.tsinghua.edu.cn}
} 
	~~
Yang Yuan\thanks{
IIIS, Tsinghua University,
e-mail: {\tt yuanyang@tsinghua.edu.cnu}}
}

\begin{document}

\date{}
\maketitle


\begin{abstract}
Semi-supervised learning has emerged as a pivotal approach for leveraging scarce labeled data alongside abundant unlabeled data. Despite significant progress, prevailing SSL methods predominantly enforce consistency between different augmented views of individual samples, thereby overlooking the rich relational structure inherent within a mini-batch. In this paper, we present \textbf{RelationMatch}, a novel SSL framework that explicitly enforces in-batch relational consistency through a Matrix Cross-Entropy (MCE) loss function. The proposed MCE loss is rigorously derived from both matrix analysis and information geometry perspectives, ensuring theoretical soundness and practical efficacy. Extensive empirical evaluations on standard benchmarks, including a notable 15.21\% accuracy improvement over FlexMatch on STL-10, demonstrate that RelationMatch not only advances state-of-the-art performance but also provides a principled foundation for incorporating relational cues in SSL.
\end{abstract}

\section{Introduction}

Semi-supervised learning (SSL) occupies a unique niche at the intersection of supervised and self-supervised learning paradigms~\citep{tian2020contrastive, chen2020simple}. By leveraging a limited set of labeled examples alongside a substantially larger corpus of unlabeled data, SSL methods combine direct label fitting with label propagation based on data manifold priors. This dual strategy has enabled SSL algorithms to achieve remarkable performance improvements over purely supervised approaches, especially in label-scarce regimes~\citep{sohn2020fixmatch, zhang2021flexmatch, wang2022freematch}.

A prevailing approach in SSL is the generation of pseudo-labels~\citep{lee2013pseudo, tschannen2019mutual, berthelot2019mixmatch, xie2020unsupervised, sohn2020fixmatch, gong2021alphamatch} by a neural network $\boldsymbol{f}$ during training. This idea, rooted in early self-training methodologies~\citep{yarowsky1995unsupervised}, involves sampling both labeled and unlabeled examples in each iteration. For unlabeled data, a weak augmentation is first applied, and the network’s prediction is evaluated; when a high confidence is attained, the prediction is adopted as a pseudo-label. The network is then trained to reproduce this label on a strongly augmented version of the same input.

While existing methods enforce consistency between the predictions for weakly and strongly augmented versions of each individual image, one may ask: Is pairwise augmentation consistency the only supervisory signal available in SSL? In practice, the relationships among different examples in a mini-batch also contain valuable information that remains largely untapped.

\begin{figure}[t] 
\centering 
\includegraphics[width=0.85\columnwidth]{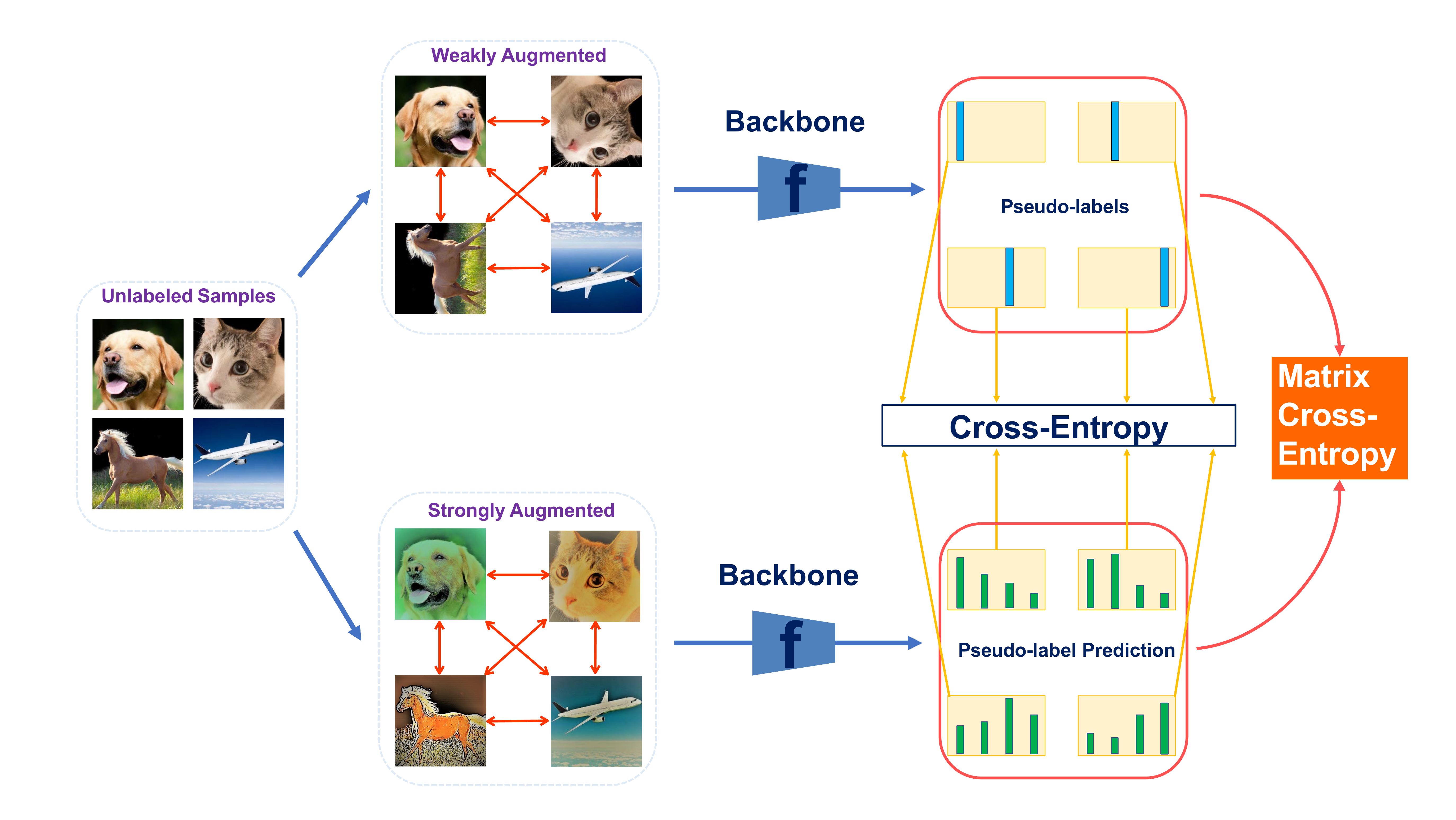}
\caption{Pseudo-labels are generated by inputting a batch of weakly augmented images into the model. Subsequently, the model computes the probability distributions for the corresponding strongly augmented images. The loss function seamlessly integrates both the traditional cross-entropy loss and the proposed matrix cross-entropy loss.}
\label{fig:arch}
\end{figure}

In this work, we propose to augment the conventional SSL framework by additionally enforcing consistency between the in-batch relationships of weak and strong augmentations. As illustrated in Figure~\ref{fig:arch}, consider a mini-batch where the upper row comprises weakly augmented images with accurate pseudo-labels, while the lower row contains their strongly augmented counterparts. Whereas existing methods only require that each strongly augmented image’s prediction closely approximates its one-hot pseudo-label, our method—termed \textbf{RelationMatch}—also requires that the relational structure among the weakly augmented images is preserved in the strongly augmented set. In other words, the relationship between a weakly augmented dog and a weakly augmented cat should be similar to that between their strongly augmented versions:
\[
\mathrm{Relation}(\text{WeaklyAug dog}, \text{WeaklyAug cat})
\approx 
\mathrm{Relation}(\text{StronglyAug dog}, \text{StronglyAug cat}).
\]
Formally, each image $\mathbf{x}$ is represented by a prediction vector $\boldsymbol{f}(\mathbf{x}) \in \mathbb{R}^k$, and the similarity between any two images is quantified via the inner product of their prediction vectors. For a mini-batch of $b$ images, the prediction matrix $\mathbf{A} \in \mathbb{R}^{b \times k}$ yields a relationship matrix defined as
\[
\operatorname{R}(\mathbf{A}) \triangleq \mathbf{A}\mathbf{A}^\top.
\]
Distinct relationship matrices are computed for weak and strong augmentations, denoted as $\operatorname{R}(\mathbf{A}^w)$ and $\operatorname{R}(\mathbf{A}^s)$, respectively.

To align these relationship matrices, we introduce the \emph{Matrix Cross-Entropy} (MCE) loss, which generalizes the standard cross-entropy loss to the matrix domain. Deriving naturally from both matrix analysis and information geometry, the MCE loss exhibits several desirable properties, including convexity and a well-defined minimization behavior. Empirically, integrating MCE within RelationMatch yields significant performance gains—for instance, achieving a 15.21\% accuracy improvement over FlexMatch on the STL-10 dataset—and also provides consistent benefits on CIFAR-10 and CIFAR-100. Moreover, our method extends gracefully to fully supervised settings.

Our contributions can be summarized as follows:
\vspace{-0.2cm}
\begin{itemize}[leftmargin=0.7cm, parsep=-0.0cm]
    \item We propose \textbf{RelationMatch}, a novel SSL framework that enforces in-batch relational consistency between weak and strong augmentations.
    \item We develop the MCE loss function, underpinned by dual theoretical perspectives from matrix analysis and information geometry, and rigorously establish its favorable properties.
    \item We conduct extensive empirical evaluations on standard vision benchmarks (CIFAR-10, CIFAR-100, and STL-10), demonstrating that RelationMatch consistently outperforms state-of-the-art methods and improves performance in both semi-supervised and supervised settings.
\end{itemize}

\section{Matrix Cross-Entropy for Supervised and Semi-supervised Learning}

\subsection{Warm-up Example}
To elucidate how our algorithm captures relationships through both weak and strong augmentations, let's begin with a straightforward example. Suppose we have $b=4, k=3$, where three out of the four images belong to the first class, and the fourth image belongs to the last class. We assume that the function $f$ assigns accurate pseudo-labels for the weak augmentations, denoting $\operatorname{R}(\mathbf{A}^w) = \mathbf{A}^w (\mathbf{A}^w)^{\top}$ as:
\begin{equation}
\operatorname{R}(\mathbf{A}^w) = \operatorname{R}\left(
\begin{bmatrix}
1 & 0 & 0\\
0 & 0 & 1\\
1 & 0 & 0\\
1 & 0 & 0
\end{bmatrix}
\right)
=
\begin{bmatrix}
1 & 0 & 1 & 1\\
0 & 1 & 0 & 0\\
1 & 0 & 1 & 1\\
1 & 0 & 1 & 1
\end{bmatrix}
\tag{2.1}
\label{eqn:one-hot-1}
\end{equation}

Since the pseudo labels for weak augmentations are always one-hot vectors, 
$\operatorname{R}(\mathbf{A}^w)$ is well structured. Specifically, 
for rows that are the same in $\mathbf{A}^w$, they are also the same in $\operatorname{R}(\mathbf{A}^w)$, with values representing the corresponding row indices. In other words, $\operatorname{R}(\mathbf{A}^w)$ represents $k$ distinct clusters of one-hot vectors in the mini-batch. 

If $f$ can generate exactly the same prediction matrix $\mathbf{A}^s$ for the strongly augmented images, our algorithm will not incur any additional loss compared with the previous cross-entropy based algorithms. However, $\mathbf{A}^s$ and $\mathbf{A}^w$ are generally different, where our algorithm becomes useful. For example, given a pair of prediction vectors $(p, q)$, 
if we know $p=(1, 0, 0)$, then cross-entropy loss is simply $p_1 \log q_1 =\log q_1$. Therefore, we will get the same loss for $q=(0.5, 0.5, 0)$,  $q=(0.5, 0.25, 0.25)$, or $q=(0.5, 0, 0.5)$. Consider the following two possible cases of $\operatorname{R}(\mathbf{A}^s)$ generated by $f$ during training:
\[
\operatorname{R}\left(
\begin{bmatrix}
0.5 & 0.5 & 0\\
0 & 0 & 1\\
0.5 & 0.5 & 0\\
0.5 & 0.5 & 0
\end{bmatrix}
\right)
=
\begin{bmatrix}
0.5 & 0 & 0.5 & 0.5\\
0 & 1 & 0 & 0\\
0.5 & 0 & 0.5 & 0.5\\
0.5 & 0 & 0.5 & 0.5
\end{bmatrix}
\]

\[
\operatorname{R}\left(
\begin{bmatrix}
0.5 & 0.5 & 0\\
0 & 0 & 1\\
0.5 & 0.25 & 0.25\\
0.5 & 0 & 0.5
\end{bmatrix}
\right)
=\begin{bmatrix}
0.5 & 0 & 0.375 & 0.25\\
0 & 1 & 0.25 & 0.5\\
0.375 & 0.25 & 0.375 & 0.375\\
0.25 & 0.5 & 0.375 & 0.5
\end{bmatrix}
\]

If we only use cross-entropy loss, these two cases will give us the same gradient information. However, by considering the in-batch relationships, it becomes clear that these two cases are different: the first case always makes mistakes on the second class, while the second case makes relatively random mistakes. Therefore, by comparing $\operatorname{R}(\mathbf{A}^s)$ with $\operatorname{R}(\mathbf{A}^w)$ defined in Eqn.~(\ref{eqn:one-hot-1}), we can get additional training signals. In our example, the first case will not give additional gradient information for the second row (and the second column due to symmetry), but the second case will. 

\subsection{Matrix Cross-Entropy}
We now introduce the Matrix Cross-Entropy (MCE) loss, which measures the discrepancy between two positive semi-definite relationship matrices, namely $\operatorname{R}(\mathbf{A}^w)$ and $\operatorname{R}(\mathbf{A}^s)$.

\begin{definition}[Matrix Cross-Entropy for Positive Semi-Definite Matrices]
\label{def:mce}
Let $\mathbf{P}$ and $\mathbf{Q}$ be positive semi-definite matrices. The \emph{Matrix Cross-Entropy} (MCE) between $\mathbf{P}$ and $\mathbf{Q}$ is defined as
\begin{equation}
\operatorname{MCE}(\mathbf{P}, \mathbf{Q}) = \operatorname{tr}\Bigl(-\mathbf{P} \log \mathbf{Q} + \mathbf{Q}\Bigr),
\end{equation}
where $\operatorname{tr}(\cdot)$ denotes the trace operator and $\log$ refers to the principal matrix logarithm (see Appendix~\ref{sec:discuss-density-matrix}).
\end{definition}

In the case of $l_2$-normalized representation vectors, the relationship matrices are further normalized. Define
\[
\mathbf{P} = \frac{1}{b}\mathbf{Y}\mathbf{Y}^{\top} \quad \text{and} \quad \mathbf{Q} = \frac{1}{b}\mathbf{X}\mathbf{X}^{\top},
\]
where each row of $\mathbf{Y}$ and $\mathbf{X}$ is $l_2$-normalized. Since the diagonal entries of $\mathbf{X}\mathbf{X}^{\top}$ are all 1, we have $\operatorname{tr}(\mathbf{X}\mathbf{X}^{\top}) = b$, and the MCE simplifies to
\begin{equation}
\operatorname{MCE}(\mathbf{P}, \mathbf{Q}) = \operatorname{tr}\Bigl(-\mathbf{P} \log \mathbf{Q}\Bigr) + 1.
\end{equation}
The derivation follows directly from substituting the normalization condition into the definition.

Although the MCE loss may initially appear complex, its formulation is justified by rigorous derivations from both matrix analysis (Section~\ref{sec:matrix-entropy}) and information geometry (Section~\ref{sec:information-geometry}). An additional interpretation based on the eigen-decomposition of matrices is provided in Appendix~\ref{sec:pca-interpretation}. Moreover, as discussed in Section~\ref{sec:discuss-mce}, the standard cross-entropy loss is revealed as a special case of MCE, further underscoring its theoretical elegance and practical utility.

\section{Matrix Cross-Entropy: Foundations and Interpretations}
\label{sec:mce}
In this section, we rigorously develop the theoretical foundations for the Matrix Cross-Entropy (MCE) loss. We examine MCE from two complementary perspectives: matrix analysis and information geometry. Additional interpretations based on eigen-decomposition are deferred to Appendix~\ref{sec:pca-interpretation}.

\subsection{Density Matrices and the Matrix Logarithm}
\label{sec:matrix-entropy}

\begin{definition}[Density Matrix on \( \mathbb{R}^{n \times n} \)]
A matrix \( \mathbf{A} \in \mathbb{R}^{n \times n} \) is called a \emph{density matrix} if it is symmetric, positive semi-definite, and has unit trace.
\end{definition}

Density matrices generalize the concept of probability distributions to matrix spaces. Their non-negative eigenvalues and unit trace naturally mirror the non-negativity and normalization constraints found in classical probability theory.

\begin{definition}[Matrix Logarithm]
\label{def:matrix-log}
The exponential of a matrix \( \mathbf{B} \) is defined by
\[
e^{\mathbf{B}} = \sum_{n=0}^{\infty} \frac{\mathbf{B}^n}{n!}.
\]
A matrix \( \mathbf{B} \) is called the \emph{matrix logarithm} of \( \mathbf{A} \) if 
\[
e^{\mathbf{B}} = \mathbf{A}.
\]
In practice, there may exist multiple logarithms for a given matrix \( \mathbf{A} \); however, the \emph{principal matrix logarithm}~\citep{higham2008functions} provides a canonical choice, especially when \( \mathbf{A} \) is a density matrix.
\end{definition}

\begin{theorem}[Principal Matrix Logarithm~\citep{higham2008functions}]
\label{thm:principal}
Let \( \mathbf{A} \in \mathbb{C}^{n \times n} \) be a matrix with no eigenvalues on \( \mathbb{R}^{-} \). Then there exists a unique matrix \( \mathbf{X} \) such that 
\[
e^{\mathbf{X}} = \mathbf{A},
\]
with all eigenvalues of \( \mathbf{X} \) lying in the strip \( \{ z \in \mathbb{C} : -\pi < \operatorname{Im}(z) < \pi \} \). We refer to \( \mathbf{X} \) as the principal logarithm of \( \mathbf{A} \) and denote it by \( \mathbf{X} = \log(\mathbf{A}) \). If \( \mathbf{A} \) is real, then its principal logarithm is also real.
\end{theorem}

\begin{proposition}
\label{prop:spectral-log}
Let \( \mathbf{A} \) be a density matrix with spectral decomposition 
\[
\mathbf{A} = \mathbf{U} \mathbf{\Lambda} \mathbf{U}^\top,
\]
where \( \mathbf{\Lambda} = \operatorname{diag}(\lambda_1,\lambda_2,\ldots,\lambda_n) \). Then, its principal logarithm is given by
\[
\log \mathbf{A} = \mathbf{U} \operatorname{diag}(\log \lambda_1, \log \lambda_2, \ldots, \log \lambda_n) \mathbf{U}^\top.
\]
\end{proposition}

\subsection{Von Neumann Entropy and Matrix Cross-Entropy}

\begin{lemma}
For any density matrix \( \mathbf{A} \), the von Neumann entropy is equivalent to the Shannon entropy computed on its eigenvalues:
\[
-\operatorname{tr}(\mathbf{A} \log \mathbf{A}) = -\sum_{i} \lambda_i \log \lambda_i.
\]
\end{lemma}

Motivated by the simplicity and favorable optimization properties of the classical cross-entropy loss, we introduce the \emph{Matrix Cross-Entropy} (MCE) as a tractable variant of the matrix (von Neumann) divergence. Specifically, when both arguments are density matrices, we define
\begin{equation}
\operatorname{MCE}_{\text{density-matrix}}(\mathbf{P}, \mathbf{Q}) = \operatorname{H}(\mathbf{P}) + \mathrm{MRE}(\mathbf{P}, \mathbf{Q}),
\end{equation}
where \( \operatorname{H}(\mathbf{P}) = -\operatorname{tr}(\mathbf{P} \log \mathbf{P}) \) denotes the matrix (von Neumann) entropy, and \( \mathrm{MRE}(\mathbf{P}, \mathbf{Q}) \) represents the matrix relative entropy.

\begin{definition}[Matrix Relative Entropy for Density Matrices]
Let \( \mathbf{P}, \mathbf{Q} \in \mathbb{R}^{n \times n} \) be density matrices. The \emph{matrix relative entropy} of \( \mathbf{P} \) with respect to \( \mathbf{Q} \) is defined as
\[
\mathrm{MRE}(\mathbf{P}, \mathbf{Q}) = \operatorname{tr}\Bigl( \mathbf{P} \log \mathbf{P} - \mathbf{P} \log \mathbf{Q} \Bigr).
\]
\end{definition}

\subsection{Information Geometrical Perspective of Matrix Cross-Entropy}
\label{sec:information-geometry}

Information geometry provides an elegant framework to extend MCE from the realm of unit-trace density matrices to arbitrary positive semi-definite matrices. According to \citet{amari2014information}, one can endow the cone of positive semi-definite matrices with a dually flat geometry by means of the Bregman divergence. For a convex function \( \phi \), the Bregman divergence between matrices \( \mathbf{P} \) and \( \mathbf{Q} \) is defined as
\[
\operatorname{D}[\mathbf{P}: \mathbf{Q}] = \phi(\mathbf{P}) - \phi(\mathbf{Q}) - \langle \nabla \phi(\mathbf{P}), \mathbf{P} - \mathbf{Q} \rangle.
\]

Choosing \( \phi(\mathbf{P}) \) as the negative matrix entropy, i.e., \( \phi(\mathbf{P}) = -\operatorname{H}(\mathbf{P}) \), we obtain the \emph{Matrix Bregman Divergence} (denoted as \( \operatorname{MD} \)):
\[
\operatorname{MD}[\mathbf{P}: \mathbf{Q}] = \operatorname{tr}\Bigl( \mathbf{P} \log \mathbf{P} - \mathbf{P} \log \mathbf{Q} - \mathbf{P} + \mathbf{Q} \Bigr).
\]
When the reference matrix \( \mathbf{P} \) is fixed, this divergence reduces naturally to the MCE formulation introduced above, thereby inheriting desirable properties from both density matrix theory and information geometry.

\begin{theorem}[Projection Theorem~\citep{amari2014information}]
\label{thm:projection}
Let \( S \) be a smooth submanifold of the space of positive definite matrices. Then, for any matrix \( \mathbf{P} \), there exists a unique projection \( \mathbf{P}_S \in S \) that minimizes the divergence \( \operatorname{D}[\mathbf{P}:\mathbf{Q}] \) over \( \mathbf{Q} \in S \). This projection is known as the \( \eta \)-geodesic projection of \( \mathbf{P} \) onto \( S \).
\end{theorem}

An immediate consequence of this result is the following minimization property for the MCE loss:

\begin{proposition}[Minimization Property]
\label{prop:optimal-point}
For any fixed density matrix \( \mathbf{P} \), the MCE loss is minimized by choosing \( \mathbf{Q} = \mathbf{P} \), i.e.,
\[
\operatorname{argmin}_{\mathbf{Q} \succ 0} \; \operatorname{MCE}(\mathbf{P}, \mathbf{Q}) = \mathbf{P}.
\]
\end{proposition}

\begin{proof}
This result follows directly from Theorem~\ref{thm:projection}, which guarantees that the divergence is minimized when the projected matrix coincides with the reference matrix \( \mathbf{P} \).
\end{proof}

\section{Unveiling the Properties of Matrix Cross-Entropy}
\label{sec:discuss-mce}

\subsection{The Scalar Cross-Entropy as a Special Case of MCE}
\label{sec:scala-ce}

In this subsection, we demonstrate that the traditional scalar cross-entropy loss is a particular instance of the proposed Matrix Cross-Entropy (MCE) loss. This connection not only bridges the two formulations but also highlights the spectral nature of density matrices, with MCE naturally capturing both self-correlations and cross-correlations among class probabilities.

Consider a collection of \( b \) pairs of \( k \)-dimensional probability vectors, denoted as 
\[
\{ (\boldsymbol{\mu}_i, \boldsymbol{\nu}_i) \}_{i=1}^{b},
\]
where 
\[
\boldsymbol{\mu}_i = \left(\mu^{(1)}_i, \ldots, \mu^{(k)}_i\right) \quad \text{and} \quad \boldsymbol{\nu}_i = \left(\nu^{(1)}_i, \ldots, \nu^{(k)}_i\right).
\]
The scalar cross-entropy between \(\boldsymbol{\mu}_i\) and \(\boldsymbol{\nu}_i\) is defined as
\[
H(\boldsymbol{\mu}_i, \boldsymbol{\nu}_i) = -\sum_{j=1}^{k} \mu^{(j)}_i \log \nu^{(j)}_i,
\]
which can equivalently be written as
\[
H(\boldsymbol{\mu}_i, \boldsymbol{\nu}_i) = -\operatorname{tr}\Bigl(\operatorname{diag}(\boldsymbol{\mu}_i) \log \operatorname{diag}(\boldsymbol{\nu}_i)\Bigr).
\]
This formulation underscores that the cross-entropy is sensitive to the spectral (diagonal) components of the corresponding density matrices.

In the specific case where the labels are one-hot encoded, let \(\mathbf{M} \in \mathbb{R}^{b \times k}\) denote the matrix whose rows are the one-hot vectors \(\boldsymbol{\mu}_i\) and \(\mathbf{N} \in \mathbb{R}^{b \times k}\) the matrix of predicted distributions \(\boldsymbol{\nu}_i\). Define
\[
\mathbf{P} = \frac{1}{b}\mathbf{I}_b \quad \text{and} \quad \mathbf{Q} = \mathbf{I}_b \circ (\mathbf{M}\mathbf{N}^\top),
\]
where \(\circ\) represents the Hadamard product. With these definitions, the averaged cross-entropy loss can be expressed as
\[
\operatorname{tr}(-\mathbf{P} \log \mathbf{Q}),
\]
thereby illustrating that the scalar cross-entropy is a special case of MCE when the matrices involved are diagonal and encode only self-correlation.

\subsection{Desirable Properties of MCE}
\label{sec:properties-mce}

The Matrix Cross-Entropy (MCE) loss is endowed with several mathematically appealing properties that make it a robust and flexible objective for learning algorithms. We summarize these properties below.

\begin{lemma}
\label{lem:density-1}
For any nonzero matrix \(\mathbf{A} \in \mathbb{R}^{m \times n}\), the matrices 
\[
\frac{\mathbf{A}\mathbf{A}^\top}{\operatorname{tr}(\mathbf{A}\mathbf{A}^\top)} \quad \text{and} \quad \frac{\mathbf{A}^\top\mathbf{A}}{\operatorname{tr}(\mathbf{A}^\top\mathbf{A})}
\]
are density matrices.
\end{lemma}

\begin{proof}
This follows directly from the singular value decomposition (SVD) of \(\mathbf{A}\), which shows that both matrices are symmetric, positive semi-definite, and have unit trace.
\end{proof}

\begin{lemma}[Joint Convexity~\citep{lindblad1974expectations}]
\label{lemma:linblad}
The matrix relative entropy is jointly convex. Specifically, for any \(t \in [0,1]\) and any density matrices \(\mathbf{X}_1\), \(\mathbf{X}_2\), \(\mathbf{Y}_1\), and \(\mathbf{Y}_2\),
\[
\mathrm{MRE}\Bigl(t \mathbf{X}_1 + (1-t) \mathbf{X}_2; \, t \mathbf{Y}_1 + (1-t) \mathbf{Y}_2\Bigr) \leq t\,\mathrm{MRE}(\mathbf{X}_1; \mathbf{Y}_1) + (1-t)\,\mathrm{MRE}(\mathbf{X}_2; \mathbf{Y}_2).
\]
\end{lemma}

\begin{proposition}[Linearity]
For any collection of density matrices \(\{\mathbf{P}_i\}\) and corresponding scalars \(\{a_i\}\), the MCE loss satisfies
\[
\mathrm{MCE}\Bigl(\sum_i a_i \mathbf{P}_i, \mathbf{Q}\Bigr) = \sum_i a_i\, \mathrm{MCE}(\mathbf{P}_i, \mathbf{Q}).
\]
\end{proposition}

\begin{proposition}[Convexity]
The MCE loss is convex in its second argument. More precisely, for any density matrix \(\mathbf{P}\) and any convex combination \(\sum_j b_j \mathbf{Q}_j\) (with \(b_j \geq 0\) and \(\sum_j b_j = 1\)), it holds that
\[
\mathrm{MCE}\Bigl(\mathbf{P}, \sum_j b_j \mathbf{Q}_j\Bigr) \leq \sum_j b_j\, \mathrm{MCE}(\mathbf{P}, \mathbf{Q}_j).
\]
\end{proposition}

\begin{proof}
This convexity property is a direct consequence of the joint convexity of the matrix relative entropy (Lemma~\ref{lemma:linblad}). Alternatively, one may arrive at the same conclusion by applying a unitary transformation and the spectral decompositions of \(\mathbf{P}\) and \(\mathbf{Q}\).
\end{proof}

\begin{proposition}[Lower Boundedness]
Let \(\mathbf{P}\) be a density matrix. Then, for any positive semi-definite matrix \(\mathbf{Q}\), the MCE loss satisfies
\[
\mathrm{MCE}(\mathbf{P}, \mathbf{Q}) \geq -\log\operatorname{tr}(\mathbf{P}\mathbf{Q}) + \operatorname{tr}(\mathbf{Q}).
\]
\end{proposition}

\begin{proof}
This lower bound can be derived by applying the spectral decompositions of \(\mathbf{P}\) and \(\mathbf{Q}\) together with standard trace inequalities.
\end{proof}

\section{RelationMatch: Applying MCE for Semi-supervised Learning}

We consider a general framework that unifies many prior semi-supervised learning algorithms (e.g., FixMatch~\citep{sohn2020fixmatch} and FlexMatch~\citep{zhang2021flexmatch}), which can be succinctly formulated as
\begin{equation}\label{update_formula}
\theta_{n+1} \leftarrow \underset{\theta}{\arg\min} \left\{ \mathcal{L}_{\text{sup}}(\theta) + \mu_u \, \mathcal{L}_u(\theta; \theta_n) \right\},
\end{equation}
where $\theta_n$ denotes the model parameters at the $n$-th iteration, $\mathcal{L}_{\text{sup}}(\theta)$ is the supervised loss, and $\mathcal{L}_u(\theta; \theta_n)$ is an unsupervised consistency regularization term that operates on the unlabeled data. The hyperparameter $\mu_u$ controls the trade-off between the two losses.

During training, both labeled and unlabeled data are available. Let $b_s$ denote the number of labeled images and $b_u$ the number of unlabeled images. For the labeled data, we employ the standard cross-entropy loss. For the subset $b_u^\prime \leq b_u$ of unlabeled data for which pseudo-labels are generated, we introduce a combined loss that incorporates both the conventional cross-entropy (CE) and our proposed Matrix Cross-Entropy (MCE) loss. Specifically, assume that the pseudo-labels from weakly augmented images are given by
\[
\tilde{\mathbf{Y}}_w = [\tilde{\mathbf{y}}_1, \dots, \tilde{\mathbf{y}}_b]^\top \in \mathbb{R}^{b \times k},
\]
and denote the prediction vectors for their strongly augmented counterparts as
\[
\tilde{\mathbf{X}}_s = [\tilde{\mathbf{x}}_1, \dots, \tilde{\mathbf{x}}_b]^\top \in \mathbb{R}^{b \times k}.
\]
The overall loss function is then defined as
\begin{equation}
\begin{aligned}
\mathcal{L}_{\text{RelationMatch}}(Y, X) &= \mathcal{L}_{\text{sup}}(Y_{\text{sup}}, X_{\text{sup}}) + \mu_u \, \mathcal{L}_u(\tilde{\mathbf{Y}}_w, \tilde{\mathbf{X}}_s) \\
&=\sum_{i=1}^{b_s}\operatorname{CE}(y_i, x_i) + \mu_u \Biggl( \sum_{i=1}^{b_u} \operatorname{CE}(\tilde{y}_i, \tilde{x}_i) + \gamma_u \cdot \operatorname{MCE}\Bigl(\operatorname{R}(\tilde{\mathbf{Y}}_w), \operatorname{R}(\tilde{\mathbf{X}}_s)\Bigr) \Biggr),
\end{aligned}
\end{equation}
where the batch-normalized relationship matrices are defined as
\[
\operatorname{R}(\tilde{\mathbf{Y}}_w) = \frac{1}{b}\tilde{\mathbf{Y}}_w\,\tilde{\mathbf{Y}}_w^\top, \quad \operatorname{R}(\tilde{\mathbf{X}}_s) = \frac{1}{b}\tilde{\mathbf{X}}_s\,\tilde{\mathbf{X}}_s^\top.
\]
These matrices capture the pairwise relationships among predictions, and the MCE loss enforces that the relational structure obtained from weak augmentations is preserved under strong augmentations.

\textbf{Note.} The technical developments in Sections~\ref{sec:mce} and~\ref{sec:discuss-mce} provide the theoretical foundations for MCE. However, an in-depth understanding of these sections is not necessary for implementing RelationMatch; readers primarily interested in practical aspects may proceed directly to the experimental results.

\subsection{Datasets}

\textbf{CIFAR-10/100.} The CIFAR-10 dataset~\citep{krizhevsky2009learning} is a well-established benchmark for image classification, comprising 60,000 images across 10 classes (with 5,000 training and 1,000 test images per class). CIFAR-100~\citep{krizhevsky2009learning} extends this benchmark to 100 classes, with 500 training images and 100 test images per class.

\noindent\textbf{STL-10.} The STL-10 dataset~\citep{coates2011stl10} is widely used in semi-supervised learning research. It contains 10 classes with 500 training images and 800 test images per class, along with 100,000 unlabeled images. Originally derived from ImageNet~\citep{deng2009imagenet}, STL-10 images are of size $3 \times 96 \times 96$ pixels.

\subsection{Experimental Details}
\label{subsec:experiment-details}
\textbf{Implementation.} Our implementation is built upon the TorchSSL framework~\citep{zhang2021flexmatch} (the official codebase for FlexMatch) and leverages PyTorch~\citep{paszke2019pytorch}. We extend TorchSSL to compute the MCE loss alongside the standard unsupervised CE loss. For further implementation details, please refer to Appendix~\ref{sec:discuss-experiment}.

\noindent\textbf{Hyperparameters.} In order to ensure a fair comparison with previous methods, we adopt the same hyperparameters as in FixMatch~\citep{sohn2020fixmatch}: specifically, we set $\gamma_u = 1$ and $\mu_u = 3 \times 10^{-3}$. Optimization is performed using SGD with a momentum of 0.9 and a weight decay of $5 \times 10^{-4}$. The learning rate is initialized at 0.03 and decayed using a cosine annealing schedule. Training is conducted with a batch size of 64, maintaining a 7:1 ratio of unlabeled to labeled data, and a confidence threshold $\tau = 0.95$ is used for pseudo-labeling.

\noindent\textbf{Baselines.} We compare RelationMatch with a range of semi-supervised learning methods, including the $\Pi$-Model~\citep{lee2013pseudo}, Pseudo-Label~\citep{lee2013pseudo}, VAT~\citep{miyato2018virtual}, MeanTeacher~\citep{tarvainen2017mean}, MixMatch~\citep{berthelot2019mixmatch}, ReMixMatch~\citep{berthelot2019remixmatch}, UDA~\citep{xie2020unsupervised}, Dash~\citep{xu2021dash}, MPL~\citep{pham2021meta}, FixMatch~\citep{sohn2020fixmatch}, and FlexMatch~\citep{zhang2021flexmatch}. Baseline results are primarily obtained from the TorchSSL framework. Furthermore, our method is compatible with recent pseudo-labeling improvements such as Curriculum Pseudo Labeling (CPL).

\subsection{Supervised Learning Results}

We first evaluate the efficacy of the MCE loss in a fully supervised setting using various backbone architectures, including WideResNet-28-2, ResNet18, and ResNet50. Training is performed for 200 epochs with a cosine learning rate scheduler and a batch size of 64. For CIFAR-10 and CIFAR-100, the relative weight $\gamma_s$ (scaling the MCE loss with respect to the CE loss) is set to 0.1 and 0.01, respectively. Table~\ref{tab:results-supervised} reports the accuracy results, showing that models trained with the MCE-augmented loss consistently outperform those using only cross-entropy or label smoothing. These results validate the versatility and effectiveness of the MCE loss in supervised settings.

\begin{table}[ht]
\centering
\caption{Supervised accuracy results on CIFAR-10 and CIFAR-100. WRN denotes WideResNet. Boldface indicates the best performance and underlined values indicate the second best.}
\label{tab:results-supervised}
\resizebox{0.75\textwidth}{!}{%
\begin{tabular}{l|cccccc}
\toprule
Dataset & \multicolumn{3}{c}{CIFAR-10} & \multicolumn{3}{c}{CIFAR-100} \\ 
\cmidrule(r){1-1}\cmidrule(lr){2-4}\cmidrule(lr){5-7}
Backbone & WRN-28-2 & ResNet18 & ResNet50 & WRN-28-2 & ResNet18 & ResNet50 \\ 
\cmidrule(r){1-1}\cmidrule(lr){2-4}\cmidrule(lr){5-7}
Only CE & 94.45{\scriptsize $\pm$0.19} & 95.08{\scriptsize $\pm$0.09} & \underline{95.32}{\scriptsize $\pm$0.18} & 76.40{\scriptsize $\pm$0.31} & 78.07{\scriptsize $\pm$0.16} & \underline{79.07}{\scriptsize $\pm$0.43} \\
With Label Smoothing & 94.72{\scriptsize $\pm$0.05} & \underline{95.25}{\scriptsize $\pm$0.13} & 95.10{\scriptsize $\pm$0.32} & \underline{76.81}{\scriptsize $\pm$0.18} & \textbf{78.41}{\scriptsize $\pm$0.21} & 78.70{\scriptsize $\pm$0.44} \\
With MCE & \textbf{94.79}{\scriptsize $\pm$0.05} & \textbf{95.31}{\scriptsize $\pm$0.08} & \textbf{95.46}{\scriptsize $\pm$0.16} & \textbf{76.92}{\scriptsize $\pm$0.17} & \underline{78.37}{\scriptsize $\pm$0.14} & \textbf{79.11}{\scriptsize $\pm$0.52} \\
\bottomrule
\end{tabular}
}
\end{table}

Unlike label smoothing~\citep{szegedy2016rethinking}—which can obscure fine-grained information in certain contexts (e.g., knowledge distillation~\citep{muller2019does})—our MCE loss preserves the one-hot characteristics of the target distribution (see Lemma~\ref{lem:one-hot-property-1}), thereby maintaining the intrinsic clustering structure of the data.

\begin{lemma}[One-hot Property Preservation]
\label{lem:one-hot-property-1}
Let \( \mathbf{Z}_1 \in \mathbb{R}^{b \times k} \) be the one-hot encoded target probabilities for a batch, and let \( \mathbf{Z}_2 \in \mathbb{R}^{b \times k} \) be the corresponding predicted probabilities. If
\[
\mathbf{Z}_1 \mathbf{Z}_1^\top = \mathbf{Z}_2 \mathbf{Z}_2^\top,
\]
then each row of \( \mathbf{Z}_2 \) is necessarily one-hot, ensuring class consistency between the target and prediction.
\end{lemma}
\begin{proof}
Since a probability vector has an $l_2$ norm equal to 1 if and only if it is one-hot, the equality of the diagonal elements of $\mathbf{Z}_2 \mathbf{Z}_2^\top$ guarantees that each row of $\mathbf{Z}_2$ is one-hot. An analysis of the off-diagonal entries completes the proof.
\end{proof}

\subsection{Semi-supervised Learning Results}

Table~\ref{tab:results-semi} presents a comprehensive comparison of RelationMatch with state-of-the-art semi-supervised learning methods on CIFAR-10, CIFAR-100, and STL-10. In particular, the RelationMatch variant incorporating Curriculum Pseudo Labeling (CPL) further refines pseudo-label quality, yielding superior performance. Across various label regimes, RelationMatch consistently outperforms FixMatch and FlexMatch. Notably, on STL-10 with only 40 labels, our method achieves a substantial reduction in error rate. Importantly, the inclusion of the MCE loss does not compromise pseudo-label quality; rather, it provides a modular component that can be seamlessly integrated into future semi-supervised learning frameworks.

\begin{table}[ht]
\centering
\caption{Error rates (100\% - accuracy) on CIFAR-10, CIFAR-100, and STL-10 for state-of-the-art semi-supervised learning methods. Boldface indicates the best performance and underlined values indicate the second best.}
\vspace{1ex}
\label{tab:results-semi}
\resizebox{\textwidth}{!}{%
\begin{tabular}{l|ccc|ccc|ccc}
\toprule
Dataset & \multicolumn{3}{c|}{CIFAR-10} & \multicolumn{3}{c|}{CIFAR-100} & \multicolumn{3}{c}{STL-10} \\ 
\cmidrule(r){1-1}\cmidrule(lr){2-4}\cmidrule(lr){5-7}\cmidrule(l){8-10}
\# Labels & 40 & 250 & 4000 & 400 & 2500 & 10000 & 40 & 250 & 1000 \\ 
\cmidrule(r){1-1}\cmidrule(lr){2-4}\cmidrule(lr){5-7}\cmidrule(l){8-10}
$\Pi$ Model~\citep{rasmus2015semi} & 74.34{\scriptsize $\pm$1.76} & 46.24{\scriptsize $\pm$1.29} & 13.13{\scriptsize $\pm$0.59} & 86.96{\scriptsize $\pm$0.80} & 58.80{\scriptsize $\pm$0.66} & 36.65{\scriptsize $\pm$0.00} & 74.31{\scriptsize $\pm$0.85} & 55.13{\scriptsize $\pm$1.50} & 32.78{\scriptsize $\pm$0.40} \\
Pseudo Label~\citep{lee2013pseudo} & 74.61{\scriptsize $\pm$0.26} & 46.49{\scriptsize $\pm$2.20} & 15.08{\scriptsize $\pm$0.19} & 87.45{\scriptsize $\pm$0.85} & 57.74{\scriptsize $\pm$0.28} & 36.55{\scriptsize $\pm$0.24} & 74.68{\scriptsize $\pm$0.99} & 55.45{\scriptsize $\pm$2.43} & 32.64{\scriptsize $\pm$0.71} \\
VAT~\citep{miyato2018virtual} & 74.66{\scriptsize $\pm$2.12} & 41.03{\scriptsize $\pm$1.79} & 10.51{\scriptsize $\pm$0.12} & 85.20{\scriptsize $\pm$1.40} & 46.84{\scriptsize $\pm$0.79} & 32.14{\scriptsize $\pm$0.19} & 74.74{\scriptsize $\pm$0.38} & 56.42{\scriptsize $\pm$1.97} & 37.95{\scriptsize $\pm$1.12} \\
MeanTeacher~\citep{tarvainen2017mean} & 70.09{\scriptsize $\pm$1.60} & 37.46{\scriptsize $\pm$3.30} & 8.10{\scriptsize $\pm$0.21} & 81.11{\scriptsize $\pm$1.44} & 45.17{\scriptsize $\pm$1.06} & 31.75{\scriptsize $\pm$0.23} & 71.72{\scriptsize $\pm$1.45} & 56.49{\scriptsize $\pm$2.75} & 33.90{\scriptsize $\pm$1.37} \\
MixMatch~\citep{berthelot2019mixmatch} & 36.19{\scriptsize $\pm$6.48} & 13.63{\scriptsize $\pm$0.59} & 6.66{\scriptsize $\pm$0.26} & 67.59{\scriptsize $\pm$0.66} & 39.76{\scriptsize $\pm$0.48} & 27.78{\scriptsize $\pm$0.29} & 54.93{\scriptsize $\pm$0.96} & 34.52{\scriptsize $\pm$0.32} & 21.70{\scriptsize $\pm$0.68} \\
ReMixMatch~\citep{berthelot2019remixmatch} & 9.88{\scriptsize $\pm$1.03} & 6.30{\scriptsize $\pm$0.05} & 4.84{\scriptsize $\pm$0.01} & 42.75{\scriptsize $\pm$1.05} & \textbf{26.03}{\scriptsize $\pm$0.35} & \textbf{20.02}{\scriptsize $\pm$0.27} & 32.12{\scriptsize $\pm$6.24} & 12.49{\scriptsize $\pm$1.28} & 6.74{\scriptsize $\pm$0.14} \\
UDA~\citep{xie2020unsupervised} & 10.62{\scriptsize $\pm$3.75} & 5.16{\scriptsize $\pm$0.06} & 4.29{\scriptsize $\pm$0.07} & 46.39{\scriptsize $\pm$1.59} & 27.73{\scriptsize $\pm$0.21} & 22.49{\scriptsize $\pm$0.23} & 37.42{\scriptsize $\pm$8.44} & 9.72{\scriptsize $\pm$1.15} & 6.64{\scriptsize $\pm$0.17} \\
Dash~\citep{xu2021dash} & 8.93{\scriptsize $\pm$3.11} & 5.16{\scriptsize $\pm$0.23} & 4.36{\scriptsize $\pm$0.11} & 44.82{\scriptsize $\pm$0.96} & 27.15{\scriptsize $\pm$0.22} & 21.88{\scriptsize $\pm$0.07} & 34.52{\scriptsize $\pm$4.30} & -- & 6.39{\scriptsize $\pm$0.56} \\
MPL~\citep{pham2021meta} & 6.93{\scriptsize $\pm$0.17} & 5.76{\scriptsize $\pm$0.24} & 4.55{\scriptsize $\pm$0.04} & 46.26{\scriptsize $\pm$1.84} & 27.71{\scriptsize $\pm$0.19} & 21.74{\scriptsize $\pm$0.09} & 35.76{\scriptsize $\pm$4.83} & 9.90{\scriptsize $\pm$0.96} & 6.66{\scriptsize $\pm$0.00} \\
\cmidrule{1-10}
FixMatch~\citep{sohn2020fixmatch} & 7.47{\scriptsize $\pm$0.28} & 5.07{\scriptsize $\pm$0.05} & 4.21{\scriptsize $\pm$0.08} & 46.42{\scriptsize $\pm$0.82} & 28.03{\scriptsize $\pm$0.16} & 22.20{\scriptsize $\pm$0.12} & 35.97{\scriptsize $\pm$4.14} & 9.81{\scriptsize $\pm$1.04} & 6.25{\scriptsize $\pm$0.33} \\
FlexMatch~\citep{zhang2021flexmatch} & \underline{4.97}{\scriptsize $\pm$0.06} & 4.98{\scriptsize $\pm$0.09} & \underline{4.19}{\scriptsize $\pm$0.01} & \underline{39.94}{\scriptsize $\pm$1.62} & \underline{26.49}{\scriptsize $\pm$0.20} & \underline{21.90}{\scriptsize $\pm$0.15} & \underline{29.15}{\scriptsize $\pm$4.16} & \underline{8.23}{\scriptsize $\pm$0.39} & \underline{5.77}{\scriptsize $\pm$0.18} \\
\cmidrule{1-10}
\textbf{RelationMatch} (Ours) & 6.87{\scriptsize $\pm$0.12} & \textbf{4.85}{\scriptsize $\pm$0.04} & 4.22{\scriptsize $\pm$0.06} & 45.79{\scriptsize $\pm$0.59} & 27.90{\scriptsize $\pm$0.15} & 22.18{\scriptsize $\pm$0.13} & 33.42{\scriptsize $\pm$3.92} & 9.55{\scriptsize $\pm$0.87} & 6.08{\scriptsize $\pm$0.29} \\
\textbf{RelationMatch} (w/ CPL) & \textbf{4.96}{\scriptsize $\pm$0.05} & \underline{4.88}{\scriptsize $\pm$0.05} & \textbf{4.17}{\scriptsize $\pm$0.04} & \textbf{39.89}{\scriptsize $\pm$1.43} & \textbf{26.48}{\scriptsize $\pm$0.18} & \textbf{21.88}{\scriptsize $\pm$0.16} & \textbf{13.94}{\scriptsize $\pm$3.76} & \textbf{8.16}{\scriptsize $\pm$0.34} & \textbf{5.68}{\scriptsize $\pm$0.19} \\
\bottomrule
Fully-Supervised & \multicolumn{3}{c|}{4.62{\scriptsize $\pm$0.05}} & \multicolumn{3}{c|}{19.30{\scriptsize $\pm$0.09}} & \multicolumn{3}{c}{--} \\
\bottomrule
\end{tabular}
}
\end{table}

\section{Related Work}

Semi-supervised learning seeks to enhance model performance by leveraging large quantities of unlabeled data, a paradigm that has attracted considerable attention in recent years~\citep{chen2020big, assran2021semi, wang2021data}. Central to many effective semi-supervised methods is the invariance principle, which posits that semantically similar inputs should yield analogous representations when processed by the same model backbone.

\paragraph{Consistency Regularization.} One of the most widely adopted approaches to enforce this invariance is consistency regularization, initially introduced in the $\Pi$-Model~\citep{Rasmus2015SemiSupervisedLW} and subsequently refined in works such as~\citep{tarvainen2017mean, laine2016temporal, berthelot2019mixmatch}. In these methods, pseudo-labels are generated for unlabeled data via appropriate data augmentations~\citep{tschannen2019mutual, berthelot2019mixmatch, xie2020unsupervised, sohn2020fixmatch, gong2021alphamatch} and are subsequently used to guide training. The conventional strategy employs an entropy minimization objective to align the pseudo-labels derived from different augmentations of the same input~\citep{Rasmus2015SemiSupervisedLW, laine2016temporal, tarvainen2017mean}. In addition, several studies have explored techniques for generating high-quality pseudo-labels that account for various practical considerations~\citep{hu2021simple, nassar2021all, xu2021dash, zhang2021flexmatch, li2022maxmatch, wang2022debiased}. More recently, SimMatch~\citep{Zheng_2022_CVPR} has extended consistency regularization by integrating contrastive learning, thereby capturing relational structures at the representation level.

\paragraph{Enhancing Pseudo-Label Quality.} Much of the literature on consistency regularization focuses on improving the quality of pseudo-labels. For example, SimPLE~\citep{hu2021simple} proposes a paired loss that minimizes the statistical discrepancy between confident and similar pseudo-labels. Adaptive strategies for pseudo-label filtering have been introduced in Dash~\citep{xu2021dash} and FlexMatch~\citep{zhang2021flexmatch} to better accommodate the dynamics of training. Other works incorporate complementary learning signals: CoMatch~\citep{li2021comatch} integrates contrastive learning to jointly optimize dual representations, while SemCo~\citep{nassar2021all} leverages external label semantics to mitigate pseudo-label degradation among visually similar classes. Additional contributions include FreeMatch~\citep{wang2022freematch}, which adjusts the confidence threshold adaptively based on the model's learning status, and MaxMatch~\citep{li2022maxmatch}, which formulates a worst-case consistency regularization with theoretical guarantees. Other approaches, such as NP-Match~\citep{wang2022np}, SEAL~\citep{tan2023seal}, and SoftMatch~\citep{chen2023softmatch}, further refine pseudo-label quality through techniques ranging from neural process modeling to the design of weighting functions that address the inherent quantity-quality trade-off in pseudo-labeling.

Collectively, these methods have established consistency regularization as a foundational component in semi-supervised learning, and our work extends this line of inquiry by proposing a novel framework that exploits in-batch relational consistency.

\section{Conclusion}
In this work, we introduced \textbf{RelationMatch}, a semi-supervised learning framework that capitalizes on the relational structure among in-batch samples via a novel Matrix Cross-Entropy (MCE) loss. By deriving MCE from complementary perspectives—matrix analysis and information geometry—we established a robust theoretical foundation that integrates seamlessly with conventional consistency regularization techniques. Extensive experiments on CIFAR-10, CIFAR-100, and STL-10 demonstrate that RelationMatch consistently outperforms existing state-of-the-art methods, yielding substantial improvements in classification accuracy while preserving the intrinsic clustering properties of the data. We anticipate that the principles underlying MCE will inspire further research into relational representations across diverse learning paradigms, and we envisage future extensions of our framework to additional modalities and more complex real-world scenarios.

\subsection*{Author Contributions}
Yifan Zhang proposed the matrix cross-entropy loss, as well as the derivations from the density matrix and information geometry perspectives. Jingqin Yang implemented most of the experiments and wrote the experimental section.
Zhiquan Tan used trace normalization to show the relationship between MRE and MCE. He also provided the PCA viewpoint of MCE, and proved most of the theorems.
Yang Yuan wrote the introduction, and provided the warm-up example. 

\vspace{10ex}
\bibliography{reference}
\bibliographystyle{iclr}

\clearpage
\appendix
\hypersetup{linkcolor=black, citecolor=mydarkblue}

\renewcommand{\appendixpagename}{\centering \LARGE Appendix}
\appendixpage

\startcontents[section]
\printcontents[section]{l}{1}{\setcounter{tocdepth}{2}}
\clearpage

\hypersetup{linkcolor=mydarkblue, citecolor=mydarkblue}

\section{More on Matrix Cross-Entropy}
\label{sec:more-mce}

As noted by \citet{kornblith2019similarity}, a robust measure of similarity between neural network representations should remain invariant under orthogonal (unitary) transformations. In this context, we state the following result.

\begin{lemma}
\label{lem:unitary-transform}
Let $\mathbf{A} \in \mathbb{R}^{n \times n}$ be a density matrix and let $\mathbf{U} \in \mathbb{R}^{m \times n}$ be a unitary matrix (i.e., $\mathbf{U}\mathbf{U}^\top=\mathbf{I}_m$). Then, the transformed matrix $\mathbf{U} \mathbf{A} \mathbf{U}^\top$
is also a density matrix.
\end{lemma}

\begin{proposition}[Invariance Property]
The Matrix Cross-Entropy (MCE) loss is invariant under simultaneous unitary transformation of its arguments. In other words, for any density matrices $\mathbf{P}$ and $\mathbf{Q}$ and any unitary matrix $\mathbf{U}$,
\begin{equation}
\mathrm{MCE}(\mathbf{P}, \mathbf{Q}) = \mathrm{MCE}\Bigl(\mathbf{U}\mathbf{P}\mathbf{U}^\top,\, \mathbf{U}\mathbf{Q}\mathbf{U}^\top\Bigr).
\end{equation}
\end{proposition}

\subsection{Density Matrices}
\label{sec:discuss-density-matrix}

We now provide additional discussion on density matrices and their connection to probability distributions. In particular, an arbitrary density matrix can be interpreted as inducing a probability distribution over any chosen orthonormal basis.

\begin{proposition}
Let $\{\mathbf{x}_i\}_{i=1}^n$ be an orthonormal basis for $\mathbb{R}^n$. Then, any density matrix $\mathbf{A} \in \mathbb{R}^{n \times n}$ induces a probability distribution over the basis elements via
\[
P(x = \mathbf{x}_i) = \mathbf{x}_i^\top \mathbf{A} \mathbf{x}_i.
\]
\end{proposition}

For ease of exposition, we define the operator $\operatorname{diag}:\mathbb{R}^n \to \mathbb{R}^{n \times n}$ as
\begin{equation*}
\operatorname{diag}(a) := \begin{pmatrix}
a_1 & 0 & \cdots & 0 \\
0 & a_2 & \cdots & 0 \\
\vdots & \vdots & \ddots & \vdots \\
0 & 0 & \cdots & a_n
\end{pmatrix}.
\end{equation*}

Conversely, given a probability distribution $\{P(x = \mathbf{x}_i)\}_{i=1}^n$ satisfying $\sum_{i=1}^n P(x = \mathbf{x}_i) = 1$, one may construct a corresponding density matrix in two distinct ways.

\begin{proposition}
\label{prop:diag}
A diagonal density matrix corresponding to the probability distribution is given by
\begin{equation}
\mathbf{A}_{\mathrm{diag}} = \begin{pmatrix}
P(x = \mathbf{x}_1) & 0 & \cdots & 0 \\
0 & P(x = \mathbf{x}_2) & \cdots & 0 \\
\vdots & \vdots & \ddots & \vdots \\
0 & 0 & \cdots & P(x = \mathbf{x}_n)
\end{pmatrix}.
\end{equation}
\end{proposition}

\begin{proposition}
\label{prop:ortho}
Let $X = \{\mathbf{x}_i\}_{i=1}^n$ be an orthonormal basis for $\mathbb{R}^n$. For any probability distribution $\{P(x = \mathbf{x}_i)\}_{i=1}^n$ with $\sum_{i=1}^n P(x = \mathbf{x}_i)=1$, define
\[
\psi = \sum_{i=1}^n \sqrt{P(x = \mathbf{x}_i)}\, \mathbf{x}_i.
\]
Then, the density matrix constructed via the orthogonal projection is
\[
\mathbf{A}_{\mathrm{op}} = \psi \psi^\top = \sum_{i,j=1}^n \sqrt{P(x = \mathbf{x}_i)P(x = \mathbf{x}_j)}\, \mathbf{x}_i \mathbf{x}_j^\top.
\]
\end{proposition}

It is straightforward to verify that
\[
P(x = \mathbf{x}_i) = \mathbf{x}_i^\top \mathbf{A}_{\mathrm{diag}} \mathbf{x}_i = \mathbf{x}_i^\top \mathbf{A}_{\mathrm{op}} \mathbf{x}_i.
\]
However, note that $\mathbf{A}_{\mathrm{diag}}$ represents a mixed state (having full rank), whereas $\mathbf{A}_{\mathrm{op}}$ represents a pure state (having rank one).

\begin{definition}[Pure State and Mixed State]
\label{def:pure-mix}
A density matrix $\mathbf{A}$ is said to represent a \emph{pure state} if $\operatorname{rank}(\mathbf{A})=1$, and a \emph{mixed state} otherwise.
\end{definition}

Furthermore, the mixedness of a density matrix can be quantified via the matrix (von Neumann) entropy.

\begin{definition}[Matrix (von Neumann) Entropy]
The matrix (von Neumann) entropy of a density matrix $\mathbf{A}$ is defined as
\[
-\operatorname{tr}(\mathbf{A} \log \mathbf{A}),
\]
which generalizes the classical notion of entropy to the matrix domain.
\end{definition}

\subsection{PCA-Inspired Interpretation}
\label{sec:pca-interpretation}

Consider positive semi-definite matrices $\mathbf{P}$ and $\mathbf{Q}$ with eigen decompositions
\[
\mathbf{P} = \mathbf{V}\mathbf{\Lambda}\mathbf{V}^\top \quad \text{and} \quad \mathbf{Q} = \mathbf{U}\Theta\mathbf{U}^\top,
\]
where $\mathbf{V}$ and $\mathbf{U}$ are orthogonal matrices, and $\mathbf{\Lambda}$ and $\Theta$ are diagonal matrices of eigenvalues. Since the matrix logarithm satisfies
\[
\log \mathbf{Q} = \mathbf{U}\,\log \Theta\, \mathbf{U}^\top,
\]
we can simplify the term
\[
\operatorname{tr}(-\mathbf{P} \log \mathbf{Q}) = \operatorname{tr}\Bigl(-\mathbf{V}\mathbf{\Lambda}\mathbf{V}^\top \mathbf{U}\,\log \Theta\, \mathbf{U}^\top\Bigr).
\]
Exploiting the cyclic property of the trace and letting $\mathbf{v}_i$ and $\mathbf{u}_j$ denote the $i$-th and $j$-th columns of $\mathbf{V}$ and $\mathbf{U}$ respectively, we obtain
\[
\operatorname{tr}(-\mathbf{P} \log \mathbf{Q}) = -\sum_{i,j} \left(\mathbf{v}_i^\top \mathbf{u}_j\right)^2 \lambda_i \log \theta_j.
\]
Thus, the overall loss expression becomes
\begin{equation}
\operatorname{tr}(-\mathbf{P} \log \mathbf{Q} + \mathbf{Q})
=-\sum_{i,j} \left(\mathbf{v}_i^\top \mathbf{u}_j\right)^2 \lambda_i \log \theta_j + \sum_{j} \theta_j.
\end{equation}
This formulation reveals that the term $\sum_{j} \theta_j$ acts as a regularizer that penalizes the eigenvalues of $\mathbf{Q}$, while the factors $(\mathbf{v}_i^\top \mathbf{u}_j)^2$ capture the correlation between the eigenvectors of $\mathbf{P}$ and $\mathbf{Q}$. In scenarios where $\mathbf{P}$ and $\mathbf{Q}$ represent covariance or correlation matrices, these correlations are closely related to the principal components in PCA.

\subsection{Analysis of the Optimal Point}
\label{sec:optimal-point}

What happens when the MCE loss is minimized? As characterized by Lemma~\ref{lem:one-hot-property-1} (see the main text), at the optimum the predicted one-hot distributions match the target distributions. In fact, one can derive the singular value decomposition for one-hot encoded data directly. Consider a (pseudo-)labeled dataset $\{(x_i,y_i)\}_{i=1}^B$, and define for each class $i$ a support vector $\mathbf{m}_i \in \mathbb{R}^B$ as
\[
m_{i,j}=\begin{cases}
1, & \text{if } y_j = i,\\[1mm]
0, & \text{otherwise}.
\end{cases}
\]
Denote the normalized vector by $\hat{\mathbf{m}}_i = \mathbf{m}_i/\|\mathbf{m}_i\|_2$ and let $\mathbf{e}_i \in \mathbb{R}^K$ be the $i$-th standard basis vector. Then, the matrix
\[
\sum_{i=1}^K \|\mathbf{m}_i\|_2\, \mathbf{e}_i \hat{\mathbf{m}}_i^\top
\]
provides the singular value decomposition for the one-hot encoded dataset. Given the close connection between singular value decomposition and eigendecomposition, analogous insights can be drawn for the corresponding correlation matrix.

\section{More Details on Experiments}
\label{sec:discuss-experiment}

\subsection{Implementing a Differentiable Matrix Logarithm}

\begin{theorem}[Taylor Series Expansion~\citep{hall2013lie}]
\label{thm:taylorhall2013}
The matrix logarithm of $\mathbf{A}$ can be expressed as
\[
\log \mathbf{A} = \sum_{m=1}^{\infty} (-1)^{m+1} \frac{(\mathbf{A} - \mathbf{I})^m}{m},
\]
which converges and is continuous for all $n \times n$ complex matrices $\mathbf{A}$ satisfying $\|\mathbf{A}-\mathbf{I}\| < 1$. Moreover, for all such $\mathbf{A}$,
\[
e^{\log \mathbf{A}} = \mathbf{A}.
\]
Similarly, for any matrix $\mathbf{X}$ with $\|\mathbf{X}\|_F < \log 2$, we have $\|e^{\mathbf{X}} - \mathbf{I}\| < 1$ and
\[
\log e^{\mathbf{X}} = \mathbf{X}.
\]
\end{theorem}

In our experiments, we consider two methods for computing the matrix logarithm:
\begin{enumerate}
    \item \textbf{Taylor Expansion:} Directly using the Taylor series expansion as given above.
    \item \textbf{Element-wise Logarithm:} Applying the logarithm function element-wise as a surrogate.
\end{enumerate}
For theoretical properties and connections between cross-entropy and matrix cross-entropy, please refer to the preceding sections. In practice, to enhance numerical stability, we regularize the matrices $\mathbf{P}$ and $\mathbf{Q}$ by adding a small multiple of the identity matrix, i.e., $\lambda \mathbf{I}$.

Table~\ref{tab:results-ablation-log} summarizes a comparison of these two methods on STL-10 using RelationMatch with Curriculum Pseudo Labeling (CPL). The Taylor expansion approach yields significantly better performance than the element-wise logarithm method. Moreover, as noted by \citet{balestriero2022contrastive}, in the self-supervised learning regime SimCLR~\citep{chen2020simple} can be reinterpreted as applying an element-wise matrix cross-entropy between relation matrices. We leave further exploration of matrix cross-entropy in self-supervised learning to future work.

\begin{table}[ht!]
\centering
\small
\caption{RelationMatch (w/ CPL) with different matrix logarithm implementations on STL-10.}
\vspace{1ex}
\label{tab:results-ablation-log}
\begin{tabular}{@{}c|cc@{}}
\toprule
\multirow{2}{*}{Method} & \multicolumn{2}{c}{STL-10} \\ \cmidrule(l){2-3} 
                        & 40 labels & 250 labels \\ \midrule
Element-wise          & $80.39 \pm 4.05$ & $89.98 \pm 0.47$ \\
\textbf{Taylor Expansion (Order 3)} & $\mathbf{86.06} \pm 3.76$ & $\mathbf{91.84} \pm 0.34$ \\ 
\bottomrule
\end{tabular}
\end{table}

\section{Comparison between SimMatch and RelationMatch}

Our approach can be viewed as a natural extension of the conventional cross-entropy (CE) loss. When employing one-hot pseudo-labels in semi-supervised settings, RelationMatch effectively captures the relational structure among predictions. In contrast, SimMatch utilizes a contrastive loss to enforce consistency of relational representations, which represents a distinct design choice.

It is important to emphasize that our method not only exhibits strong theoretical foundations (as detailed in Sections~\ref{sec:mce} and~\ref{sec:discuss-mce}) but also can be easily integrated into existing semi-supervised frameworks. In our experiments on CIFAR-10, RelationMatch outperforms SimMatch (a detailed comparison is provided in the supplementary material), thus highlighting the practical advantages of our approach. For the sake of fairness, we have not included SimMatch as a baseline in the current version. We plan to incorporate comprehensive experiments with SimMatch as a baseline in future work when sufficient computational resources become available.

\end{document}